\pdfoutput=1

\documentclass[conference]{IEEEtran}
\ifCLASSINFOpdf
\else
\fi
\usepackage{fancyhdr}
\usepackage{balance}
\usepackage{framed}
\usepackage{epsfig}
\usepackage{color}
\usepackage{subfigure}
\usepackage{multirow,tabularx}
\usepackage{amssymb}
\usepackage{mathtools}
\usepackage{graphicx}
\usepackage{epstopdf}
\usepackage{cite}
\usepackage{multirow}
\usepackage{bm}
\usepackage{hhline}
\usepackage{csquotes}
\usepackage{array}
\usepackage[yyyymmdd,hhmmss]{datetime}
\usepackage[subject={Todo}]{pdfcomment}
\usepackage[textsize=footnotesize,bordercolor=black!20]{todonotes}
\usepackage{xcolor,colortbl}
\usepackage{amssymb}
\usepackage{pifont}
\usepackage[algo2e,titlenumbered,ruled]{algorithm2e} 
\usepackage{lipsum,environ,amsmath}
\usepackage{slashbox,amsmath}
\usepackage{xspace}

\usepackage{hyperref}
\hypersetup{draft}

\usepackage{bbm}

\newcounter{Lcount}
\newcommand{\numsquishlist}{
   \begin{list}{\arabic{Lcount}. }
    { \usecounter{Lcount}
 \setlength{\itemsep}{-.1ex}      \setlength{\parsep}{0ex}
      \setlength{\topsep}{0ex}       \setlength{\partopsep}{0ex}
      \setlength{\leftmargin}{1em} \setlength{\labelwidth}{1em}
      \setlength{\labelsep}{0.1em} } }
\newcommand{\numsquishend}{\end{list}}

\newcommand{\squishlist}{
   \begin{list}{$\bullet$}
    { \setlength{\itemsep}{-.1ex}      \setlength{\parsep}{0ex}
      \setlength{\topsep}{0ex}       \setlength{\partopsep}{0ex}
      \setlength{\leftmargin}{.8em} \setlength{\labelwidth}{1em}
      \setlength{\labelsep}{0.5em} } }
\newcommand{\squishend}{\end{list}}

\newcommand{\clip}{{\sc Coordination Initiator Inference Problem}\xspace}%

\newcommand{\argmax}{\mathop{\mathrm{argmax}}\limits} 
\newcommand{\indep}{\rotatebox[origin=c]{90}{$\models$}}

\usepackage{setspace}

\usepackage{eso-pic}

\makeatletter

\def\@IEEEpubidpullup{8\baselineskip}
\makeatother

\begin{document}


%


\title{Variable-lag Granger Causality \\for Time Series Analysis}



%


\author{ 
\IEEEauthorblockN{Chainarong Amornbunchornvej\IEEEauthorrefmark{1}, Elena Zheleva\IEEEauthorrefmark{2}, and
Tanya Y. Berger-Wolf\IEEEauthorrefmark{3}}

\IEEEauthorblockA{
\IEEEauthorrefmark{1}National Electronics and Computer Technology Center, Pathum Thani, Thailand \\
Email: chainarong.amo@nectec.or.th\IEEEauthorrefmark{1}}
\IEEEauthorblockA{
\IEEEauthorrefmark{2}\IEEEauthorrefmark{3}Department of Computer Science, University of Illinois at Chicago, Chicago, IL, USA. \\
Email: ezheleva@uic.edu\IEEEauthorrefmark{2}, tanyabw@uic.edu\IEEEauthorrefmark{3}}
}


\maketitle
\begin{abstract}
  Granger causality is a fundamental technique for causal inference in time series data, commonly used in the social and biological sciences. Typical operationalizations of Granger causality make a strong assumption that every time point of the effect time series is influenced by a combination of other time series with a fixed time delay. However, the assumption of the fixed time delay does not hold in many applications, such as collective behavior, financial markets, and many natural phenomena. To address this issue, we develop variable-lag Granger causality, a generalization of Granger causality that relaxes the assumption of the fixed time delay and allows causes to influence effects with arbitrary time delays. In addition, we propose a method for inferring variable-lag Granger causality relations.  We demonstrate our approach on an application for studying coordinated collective behavior and show that it  performs better than several existing methods in both simulated and real-world datasets. Our approach can be applied in any domain of time series analysis. 
\end{abstract}


%
\IEEEpeerreviewmaketitle

\section{Introduction}

Inferring causal relationships from data is a fundamental problem in statistics, economics, and science in general. The gold standard for assessing causal effects is running randomized controlled trials which randomly assign a treatment (e.g., a drug or a specific user interface) to a subset of a population of interest, and randomly select another subset as a control group which is not given the treatment, thus attributing the outcome difference between the two groups to the treatment. However, in many cases, running such trials may be unethical, expensive, or simply impossible~\cite{Varian7310}. 
To address this issue, several methods have been developed to estimate causal effects from observational data~\cite{pearl2000causality,Spirtes1993}. 

In the context of time series data, a well-known method that defines a causal relation in terms of \emph{predictability} is Granger causality~\cite{granger1969investigating}. $X$ Granger-causes $Y$ if past information on $X$ predicts the behavior of $Y$ better than $Y$'s past information alone~\cite{Arnold:2007:TCM:1281192.1281203}. In this work, when we refer to causality, we mean specifically the predictive causality defined by Granger causality. The key assumptions of Granger causality are that 1) the process of effect generation can be explained by a set of structural equations, and 2) the current realization  of the effect at any time point is influenced by a set of causes in the past. Similar to other causal inference methods, Granger causality assumes unconfoundedness and that all relevant variables are included in the analysis. 
There are several studies that have been developed based on Granger causality ~\cite{liu2012sparse,atukeren2010relationship,peters2013causal}. The typical operational definitions~\cite{atukeren2010relationship} and inference methods for inferring Granger causality, including the common software implementation packages~\cite{MLSourcecode,RSourcecode}, assume that the effect is influenced by the cause with a fixed and constant time delay. 

This assumption of a fixed and constant time delay between the cause and effect is, in fact, too strong for many applications of understanding natural world and social phenomena. In such domains, data is often in the form of a set of time series and a common question of interest is which time series are the (causal) initiators of patterns of behaviors captured by another set of time series. For example, who are the individuals who influence a group's direction in collective movement? What are the sectors that influence the stock market dynamics right now? Which part of the brain is critical in activating a response to a given action? In all of these cases, effects follow the causal time series with delays that can vary over time. 


To address the remaining gap, we introduce the concept \emph{Variable-lag Granger causality} and a method to infer it in time series data. We prove that our definition and the proposed inference method can address the arbitrary-time-lag influence between cause and effect, while the traditional operationalizations of Granger causality and the corresponding inference methods cannot. We show that the traditional definition is indeed a special case of the new relation we define. We demonstrate the applicability of the newly defined causal inference framework by inferring initiators of collective coordinated movement, a problem proposed in~\cite{FLICAtkdd}. 

We use Dynamic Time Warping (DTW)~\cite{sakoe1978dynamic} to align the cause $X$ to the effect time series $Y$ while leveraging the power of Granger causality. In the literature, there are many clustering-based Granger causality methods that use DTW to cluster time series and perform Granger causality only for time series within the same clusters~\cite{yuan2016deep,Peng:2007:SSI:1288552.1288557}. Previous work on inferring causal relations using both Granger causality and DTW has the assumption that the smaller warping distance between two time series, the stronger the causal relation is~\cite{sliva2015tools}. If the minimum distance of elements within the DTW optimal warping path is below a given distance threshold, then the method considers that there is a causal relation between the two time series. However, their work assumes that Granger causality and DTW should run independently.
In contrast, our method formalizes the integration of Granger causality and DTW by generalizing the definition of Granger causality itself and using DTW as an instantiation of the optimal alignment requirement of the time series. 

In addition to the standard uses of Granger causality, our method is capable of:
\squishlist
\item {\bf Inferring arbitrary-lag causal relations:} our method can infer Granger causal relation where a cause influences an effect with arbitrary delays that can change dynamically; 
\item {\bf Quantifying variable-lag emulation:} our method can report the similarity of time series patterns between the cause and the delayed effect, for arbitrary delays; 
\squishend

We also prove that when multiple time series cause the behavioral convergence of a set of time series then we can treat the set of these initiating causes in the aggregate and there is a causal relation between this aggregate cause (of the set of initiating time series) and the aggregate of the rest of the time series. We provide  many experiments and examples using both simulated and real-world datasets to measure the performance of our approach in various causality settings and discuss the resulting domain insights. Our framework is highly general and can be used to analyze time series from any domain.

\section{Related work}


Many causal inference methods assume that the data is {\em i.i.d.} and rely on knowing a mechanism that generates that data, e.g., expressed through causal graphs or structural equations~\cite{pearl2000causality}. 
In time series data, the values of the consecutive time steps violate the {\em i.i.d.} assumption.
Another set of causal inference methods relax the strong {\em i.i.d} assumption, and instead assume independence between the cause 
and the mechanism 
generating the effect 
~\cite{janzing2010causal,scholkopf2012causal,shajarisales2015telling}. Specifically, knowing the cause $X=x$ never reveals information about the structural function $f(X)$ and vice versa. This idea has been used in the context of times series data~\cite{shajarisales2015telling} by relying on the concept of Spectral Independence Criterion (SIC). If a cause $X$ is a stationary process that generates the effect $Y$ via linear time invariance filter $h$ (mechanism), then $X$ and $h$ should not contain any information about each other but dependency between them and $Y$ exists in spectral sense. 

Granger causality has inspired a lot of research since its introduction in 1969~\cite{granger1969investigating}. Recent work on Granger causality has focused on various generalizations for it, including ones based on information theory, 
such as transfer entropy~\cite{schreiber-prl00,shibuya-kdd09} and directed information graphs~\cite{7273888}. Recent inference methods are able to deal with missing data~\cite{iseki-aaai19} and enable feature selection~\cite{Sun2015}. Granger causality has even been explored as a method to offer explainability of machine learning models~\cite{schwab-aaai19}. 
However, none of them study tests for variable-lag Granger causality, as we propose in this work. Besides, no method studies a causal structure that is unstable overtime~\cite{doi:10.1098/rsta.2011.0613}. 
In our work, we also relax the stationary assumption of time series.


\newtheorem{definition}{Definition}
\newtheorem{theorem}{Theorem}[section]
\newtheorem{lemma}[theorem]{Lemma}
\newtheorem{proposition}[theorem]{Proposition}
\newtheorem{corollary}[theorem]{Corollary}

\newenvironment{proof}[1][Proof]{\begin{trivlist}
\item[\hskip \labelsep {\bfseries #1}]}{\end{trivlist}}
\newenvironment{example}[1][Example]{\begin{trivlist}
\item[\hskip \labelsep {\bfseries #1}]}{\end{trivlist}}
\newenvironment{remark}[1][Remark]{\begin{trivlist}
\item[\hskip \labelsep {\bfseries #1}]}{\end{trivlist}}

\newcommand{\qed}{\nobreak \ifvmode \relax \else
      \ifdim\lastskip<1.5em \hskip-\lastskip
      \hskip1.5em plus0em minus0.5em \fi \nobreak
      \vrule height0.75em width0.5em depth0.25em\fi}
      
\section{Granger causality and fixed lag limitation}
\begin{figure}[t!]
\centering
\includegraphics[width=.7\columnwidth]{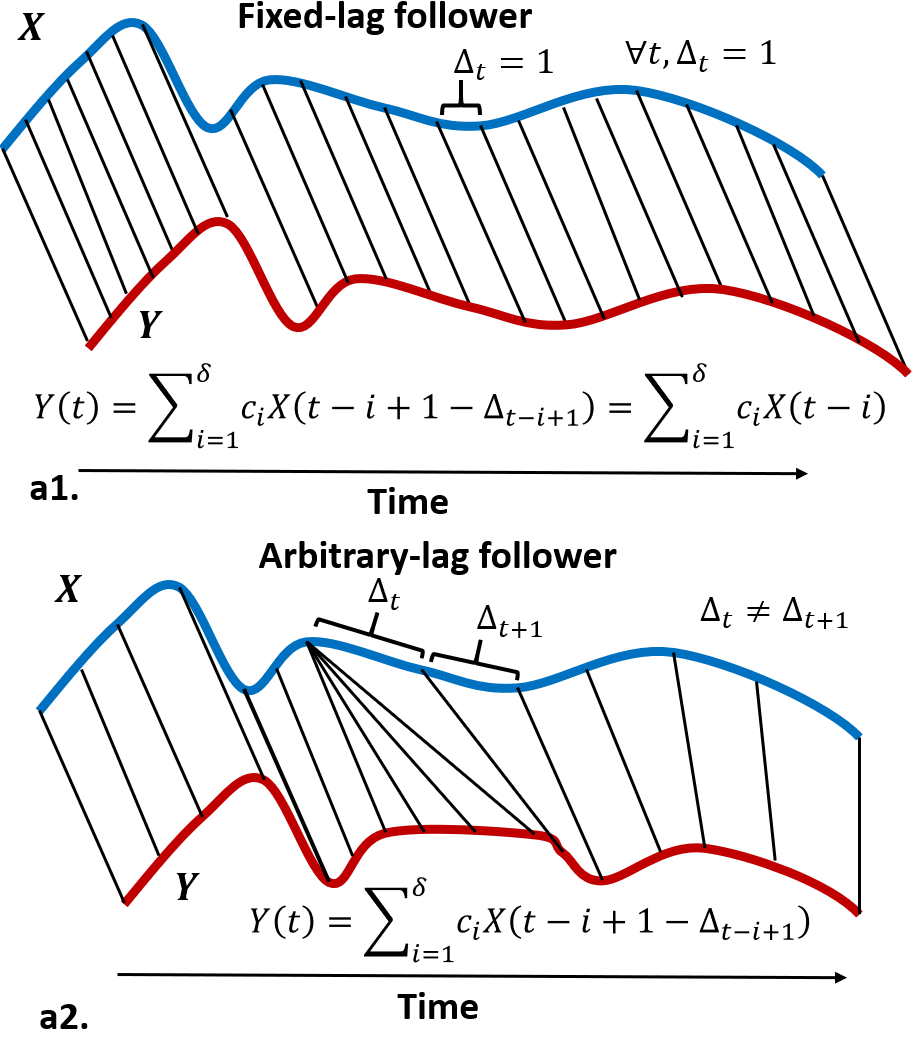}
\caption{(a1-2.) A leader (blue) influences a follower (red) at a specific time point via black lines. (a1.) The follower is a distorted version of a leader with a fixed lag.  (a2.) The follower  is a distorted version of a leader with non-fixed lags in that violates an assumption of Granger causality. Granger causality can handle only the former case and typically fails to handle later case.  We propose the generalization of Granger causality to handle variable-lag situation (equation in a2.).  }
\label{fig:LinearVSArbLag}
\end{figure}

Let $X=(X(1),\dots,X(t),\dots)$ be a time series. We will use $X(t)\in \mathbb{R}$ to denote an element of $X$ at time $t$. Given two time series $X$ and $Y$, it is said that  $X$ Granger-causes~\cite{granger1969investigating}  $Y$ if the information of $X$ in the past helps  improve the prediction of the behavior of $Y$, over $Y$'s past information alone~\cite{Arnold:2007:TCM:1281192.1281203}. The typical way to operationalize this general definition of  Granger causality ~\cite{atukeren2010relationship} is to define it as follows:
\begin{definition}[Granger causal relation]
\label{def:GC}
Let $X$ and $Y$ be time series, and $\delta_{max}\in\mathbb{N}$ be a maximum time lag. We define two residuals of regressions of $X$ and $Y$,  $r_{Y}, r_{YX}$, below: 
\begin{equation}\label{eq1}
	r_{Y}(t) = Y(t) - \sum_{i=1}^{\delta_{max}}a_i Y(t-i),
\end{equation}
\begin{equation}\label{eq2}
	r_{YX}(t) = Y(t)- \sum_{i=1}^{\delta_{max}} (a_i Y(t-i) +  b_i X(t-i)),
\end{equation}
where $a_i$ and $b_i$ are constants that optimally minimize the residual from the regression. Then $X$ Granger-causes $Y$ if the variance of $r_{YX}$ is less than  the variance of $r_Y$.
\end{definition}

This definition assumes that, for all $t>0$, $Y(t)$ can be predicted by the fixed linear combination of $a_1Y(1),\dots,a_{t-\Delta}Y(t-\Delta)$ and $b_1X(1),\dots,b_{t-\Delta}X(t-\Delta)$ with some fixed $\Delta>0$ and every $a_i,b_i$ is a fixed constant over time~\cite{atukeren2010relationship,Arnold:2007:TCM:1281192.1281203}. However, in reality, two time series might influence each other with a sequence of arbitrary, non-fixed time lags. For example, Fig.~\ref{fig:LinearVSArbLag}(a2.) has $X$ as a cause time series and $Y$ as the effect time series that imitates the values of $X$ with arbitrary lags.  Because $Y$ is not affected by $X$ with a fixed lags and the linear combination above can  change over time,  the standard Granger causality tests cannot appropriately infer Granger-causal relation between $X$ and $Y$ even if $Y$ is just a slightly distorted version of $X$ with some lags. For a concrete example, consider a movement context where time series represent  trajectories. Two people follow each other if they move in the same trajectory. Assuming the followers follow leaders with a fixed lag means the followers walk lockstep with the leader, which is not the natural way we walk.  Imagine two people embarking on a walk. The first starts walking, the second catches up a little later. They may walk together for a bit, then the second stops to tie the shoe and catches up again. The delay between the first and the second person keeps changing, yet there is no question the first sets the course and is the cause of the second's choices where to go. Fig.~\ref{fig:LinearVSArbLag} illustrates this example. 


\section{Variable-lag Granger Causality}

Here, we propose the concept of variable-lag Granger causality,  \emph{VL-Granger causality} for short, which generalizes the Granger causal relation of Definition~\ref{def:GC}  in a way that addresses the fixed-lag limitation. We demonstrate the application of the new causality relation for a specific application of inferring initiators and followers of collective behavior.

 \begin{definition}[Alignment of time series ]
 \label{def:AlignSeq}
 An alignment between two time series $X$ and $Y$ is a sequence of pairs of indices $(t_i,t_j)$, aligning $X(t_i)$ to $Y(t_j)$, such that for any two pairs in the alignment $(t_{i},t_{j})$ and $(t'_{i},t'_{j})$, if $t_{i}< t'_{i}$ then $t_{j} < t'_{j}$ (non-crossing condition). The alignment defines a sequence of delays $P=(\Delta_1,\dots,\Delta_t,\dots)$, where $\Delta_t\in\mathbb{Z}$ and $X(t-\Delta_t)$ aligns to $Y(t )$.
 \end{definition}

\begin{definition}[VL-Granger causal relation]
\label{def:ArbCausalR}
Let $X$ and $Y$ be time series, and $\delta_{max}\in\mathbb{N}$ be a maximum time lag (this is an upper bound on the time lag between any two pairs of time series values to be considered as causal). We define residual $r^*_{YX}$ of the regression:
\begin{equation}\label{eq3}
	r^*_{YX}(t) = Y(t)- \sum_{i=1}^{\delta_{max}} (a_i Y(t-i) +  b_i X(t-i) + c_i X^*(t-i)).
\end{equation}
Here $X^*(t-i)=X(t-i+1-\Delta_{t-i+1} )$, where $\Delta_{t} > 0$ is a time delay constant in  the optimal alignment sequence $P^*$ of $X$ and $Y$ that minimizes the residual of the regression.  The constants $a_i,b_i$, and $c_i$ optimally minimize the residuals $r_Y$, $r_{YX}$, and $r^*_{YX}$, respectively. The terms $b_i$ and $c_i$ can be combined but we keep them separate to clearly denote the difference between the original and proposed VL-Granger causality.
We say that $X$ VL-Granger-causes $Y$ if the variance of $r^*_{YX}$ is less than  the variances of both $r_Y$ and $r_{YX}$.
\end{definition} 

In order to make Definition~\ref{def:ArbCausalR} fully operational for this more general case (and to find the optimal constants values), we need a similarity function between two sequences which will define the optimal alignment. We propose such a similarity-based approach in Definition~\ref{def:VLemulation}.  Before defining this approach, we show that VL-Granger causality is the proper generalization of the traditional operational definition of Granger causality stated in Definition~\ref{def:GC}.
Clearly, if all the delays are constant then $r^*_{YX}(t)=r_{YX}(t)$.
\begin{proposition}
\label{prop:VLGrC1}
Let $X$ and $Y$ be time series and $P$ be their alignment sequence. If   $\forall t,\Delta_t=\Delta$, then $r^*_{YX}(t)=r_{YX}(t)$.
\end{proposition}


We must also show that the variance of $r^*_{YX}(t)$ is no greater than the variance of $r_{YX}(t)$.
\begin{proposition}
\label{prop:VLGrC2}
Let $X$ and $Y$ be time series, $P=(\Delta_1, \dots,\Delta_t, \dots)$ be their alignment sequence such that $Y(t)=X(t-\Delta_t)$. If  $\exists \Delta_t,\Delta_{t'}\in P$,  such that $\Delta_t \neq \Delta_{t'}$ and  $\forall t, X(t)\neq X(t-1)$, then   $VAR(r^*_{YX}) < VAR(r_{YX} )$.
 \end{proposition}
\begin{proof}
Because $Y(t)=X(t-\Delta_t)$, by setting $a_i=0,b_i=0,c_i=1$ for all $i$, we have  $r^*_{YX}=0$. In contrast, suppose $\Delta_{t+1}=\Delta_t+1$ and $X(t-\Delta_t-1)\neq X(t-\Delta_t) \neq X(t-\Delta_t+1)$, so $Y(t)=Y(t+1)=X(t-\Delta_t)$. Because  $a_i,b_i$ must be constant for all time step $t$ to compute $r_{YX}(t)$, at time $t$, the regression must choose to match either 1) $Y(t) - X(t-\Delta_t)=0$ and $Y(t+1) - X(t+1-\Delta_t)\neq 0$ or 2) $Y(t) - X(t-\Delta_{t+1})\neq 0$ and $Y(t+1) - X(t+1-\Delta_{t+1})=0$. Both 1) and 2) options make  $r_{YX}(t) +  r_{YX}(t+1) >0$. Hence, $VAR(r^*_{YX}) < VAR(r_{YX} )$.
\end{proof}
According to Propositions \ref{prop:VLGrC1} and \ref{prop:VLGrC2}, VL-Granger causality is the generalization of the Def.~\ref{def:GC} and always has lower or equal variance.


Of a particular interest is the case when there is an explicit similarity relation defined over the domain of the input time series. The underlying alignment of  VL-Granger causality then should incorporate that similarity measure and the methods for inferring the optimal alignment for the given similarity measure.

\begin{definition}[Variable-lag emulation]
\label{def:VLemulation}
Let $\mathcal{U}$ be a set of time series, $X,Y\in \mathcal{U}$, and $\mathrm{sim}: \mathcal{U}\times \mathcal{U} \to [0,1]$ be a similarity measure between two time series.
For a threshold $\sigma \in (0,1]$, if there exists a sequence of numbers $P=(\Delta_1,\dots,\Delta_t,\dots)$ s.t. $\mathrm{sim}(X',Y) \geq \sigma$ when $X'(t)=X(t-\Delta_t)$, then we use the following notation: 
\squishlist 
  \item if $\forall \Delta_t\in P, \:\Delta_t \geq 0$ , then $Y$ emulates $X$, denoted by  ${X \preceq Y}$,
  \item if $\forall \Delta_t\in P,  \:\Delta_t \leq  0$ , then $X$ emulates $Y$, denoted by ${Y \preceq X}$,
  \item if  ${X \preceq Y}$ and ${Y \preceq X}$, then ${Y \equiv X}$.
\squishend
We denote ${X \prec Y}$ if ${X \preceq Y}$ and $\exists \Delta_t\in P, \Delta_t > 0$.
\label{def:follr}
\end{definition}


Note, here the $\mbox{sim}$ similarity function does not have to be a distance function that obeys, among others, a triangle inequality. It can be any function that quantitatively compares the two time series. For example, it may be that when one time series increases the other decreases. We provide a more concrete and realistic example in the application setting below.

Adding this similarity measure to Definition~\ref{def:ArbCausalR} allows us to instantiate the notion of the optimal alignment $P^*$ as the one that maximizes the similarity between $X$ and $Y$: $$P^*=\argmax_P \mathrm{sim}(X',Y),$$ where  $X'(t)=X(t-\Delta_t)$ for any given $P$ and  $\Delta_t \in P$.  With that addition, if $X\preceq Y$, then $X$ VL-Granger-causes $Y$. This allows us to operationalize VL-Granger causality by checking for variable-lag emulation, as we describe in the next section.

\subsection{Example application: Initiators and followers}

In this section, we demonstrate an application of the VL-Granger causal relation to finding initiators of collective behavior. The Variable-lag emulation concept corresponds to a relation of following  in the leadership literature~\cite{FLICAtkdd}. That is,  $X\prec Y$ if $Y$ is a {\em follower} of $X$.  We are interested in the phenomenon of group convergence to a  consensus behavior and answering the question of which subset of individuals, if any, initiated that collective consensus behavior.
With that in mind,  we now define the concept of an initiator and provide a set of subsidiary definitions that allow us to formally show  (in Proposition~\ref{prop:LFprop}) that initiators of  collective behavior are indeed the time series that VL-Granger-cause the collective pattern in the set of the time series. In order to do this, we generalize our two-time series definitions to the case of multiple time series by defining the notion of an aggregate time series, which is consistent with previous Granger causality generalizations to multiple time series~\cite{siggiridou2016granger,eichler2013causal,chen2004analyzing}.  


\begin{definition} [Initiators] Let $\mathcal{U}=\{U_{1},\dots,U_{n}\}$ be a set of time series. We say that $\mathcal{X}\subseteq\mathcal{ U}$ is a set of initiators if $\forall U\in\mathcal{U\setminus\mathcal{X}}$,  $\: \exists X \in \mathcal{X}$,  $\: s.t. X \prec U$, and, conversely,  $\forall X\in  \mathcal{X} \: \exists U\in \mathcal{U\setminus\mathcal{X}},  \:  s.t.  X \prec U$. That is, every time series follows some initiator and every initiator has at least one follower.
\end{definition}

Given a set of time series $\mathcal{U}=\{U_{1},\dots,U_{n}\}$ , and a set of time series $\mathcal{X}\subseteq\mathcal{U}$, we can define an aggregate time series as a time series of means at each step:

\begin{equation}
agg(\mathcal{X})=\left({\frac{1}{|\mathcal{X}|}\sum_{U\in\mathcal{X}}U(0),\dots,\frac{1}{|\mathcal{X}|}\sum_{U\in\mathcal{X}}U(t),\dots}\right)
\end{equation}

In order to identify the state of reaching a collective consensus of a time series, while allowing for some noise, we adopt the concept of  $\epsilon$-convergence  from~\cite{doi:10.1137/100791671}.

\begin{definition} [$\epsilon$-convergence] Let $Q$ and $U$ be time series, $dist:\mathbb{R}^2\times [0,1]$ be a distance function,  and $0<\epsilon\leq1/2$.  If for all time $t\in [t_{0},t_{1}], \: dist(Q(t),U(t))\leq\epsilon$, then $Q$ and $U$ $\epsilon$-converge toward each other in the interval $[t_0,t_1]$. If $t_1=\infty$ then we  say that $Q$ and $U$ $\epsilon$-converge at time $t_0$. 
\end{definition}


\begin{definition} [$\epsilon$-convergence coordination set] Given a set of time series $\mathcal{U}=\{U_{1},\dots,U_{n}\}$, if all time series in $\mathcal{U}$ $\epsilon$-converge toward $agg(\mathcal{U})$, then we say that the set $\mathcal{U}$ is an $\epsilon$-convergence coordination set.
\end{definition}

We are finally ready to state the main connection between initiation of collective behavior and VL-Granger causality.

\begin{proposition}
\label{prop:LFprop}
Let  $dist:\mathbb{R}^2\times [0,1]$ be a distance function, $\mathcal{U}$ be a set of time series, and  $\mathcal{X}\subseteq  \mathcal{U}$ be a set of initiators, which is an $\epsilon$-convergence coordination set converging towards $agg(\mathcal{X})$ in the interval $[t_0,t_1]$. For any $U, U'\in \mathcal{U}$ of length $T$, let  $$\mathrm{sim}(U,U')=\frac{\sum_t 1- dist(U(t),U'(t))}{T}. $$ 
If  for any  $U, U' \in \mathcal{U}$ their similarity $\mathrm{sim}(U,U')\geq 1-\epsilon$ in the interval $[t_0,t_1]$, then $agg(\mathcal{X})$  VL-Granger-causes $agg(\mathcal{U}\setminus\mathcal{X})$ in that interval.
\end{proposition}
 \begin{proof}

  Suppose  $\forall X\in \mathcal{X}$, $X$ and $agg(\mathcal{X})$ $\epsilon$-converge toward each other in the interval $[t_0,t_1]$, then, by definition, for all the times $t \in [t_{0},t_1], \: dist(agg(\mathcal{X})(t),X(t))\leq\epsilon$. By the definition of initiators,  $\forall U\in\mathcal{U}\setminus\mathcal{X}, \: \exists X \in \mathcal{X}$, such that $X\prec U$, from some time $t_2 > t_0$. Thus, we have $\forall t$, s. t.  $t_2 \leq t\leq t_1, \: \: dist(X(t),U(t))\leq\epsilon$, which means $dist(agg(\mathcal{X}),U(t))\leq 2\epsilon$. Hence, we have    $\forall t, t_2\leq t \leq t_1, \:\: dist(agg(\mathcal{X})(t),agg(\mathcal{U}\setminus\mathcal{X})(t) )\leq 2\epsilon$. Since $agg(\mathcal{X})$ $2\epsilon$-converges towards some constant line $v$ in the interval $[t_0,t_1]$ and $agg(\mathcal{U}\setminus\mathcal{X})(t) )$  $2\epsilon$-converges towards the same line $v$ in the interval $[t_2, t_1]$, hence $agg(\mathcal{X}) \prec agg(\mathcal{U}\setminus\mathcal{X})$, which means, by definition, that  $agg(\mathcal{X})$  VL-Granger-causes $agg(\mathcal{U}\setminus\mathcal{X})$.
 \end{proof}
 
We have now shown that a subset of time series are initiators of a pattern of collective behavior  of an entire set if that subset VL-Granger-causes the behavior of the set. Thus, VL-Granger causality can solve the  \clip~\cite{FLICAtkdd}, which is a problem of determining  whether a pattern of collective behavior was spurious or instigated by some subset of initiators and, if so, finding those initiators who initiate collective patterns that everyone follows.

\section{VL-Granger Causality Inference}

Given a target time series $Y$, a candidate causing time series $X$, a similarity threshold $\sigma$, a significance level $\alpha$, and the max lag $\delta_{max}$, our framework evaluates whether $X$ VL-Granger causes $Y$ (with a variable lag), $X$ Granger causes $Y$ (with a fixed lag) or no conclusion of causation between $X$ and $Y$. 

 In Algorithm~\ref{algo:MainFunc} line 1-2,  we have a fix-lag parameter $FixLag$ that controls whether we choose to compute the normal Granger causality ($FixLag=true$) or VL-Granger causality ($FixLag=false$). We present the high level logic of the algorithm. However, the actual implementation is more efficient by removing the redundancies of the presented logic.

 First, we compute Granger causality (line 1 in  Algorithm~\ref{algo:MainFunc}). The flag $GrangerResult=true$ if $X$ Granger-causes $Y$, otherwise $GrangerResult=false$. Second, we  compute VL-Granger causality (line 2 in  Algorithm~\ref{algo:MainFunc}). The flag $VLGrangerResult=true$ if $X$ VL-Granger-causes $Y$, otherwise $VLGrangerResult=false$. 
 
  Based on the work in~\cite{atukeren2010relationship}, we use the Bayesian Information Criteria  (BIC) to compare the residual of regressing $Y$ on $Y$ past information, $r_Y$, with the residual of regressing $Y$ on $Y$ and $X$ past information $r_{YX}$. We use $v_1 \ll v_2$ to represent that $v_1$ is less than $v_2$ with statistical significance by using some statistical test(s). If $BIC(r_Y)\ll BIC(r_{YX})$, then we conclude that the prediction of  $Y$  using $Y,X$ past information is better than the prediction of  $Y$  using $Y$ past information alone. After we got the results of both  $GrangerResult$ and $VLGrangerResult$, then we proceed to report the conclusion of causal relation between $X$ and $Y$ w.r.t.  the following four conditions.
 
 \squishlist
\item {\bf If both $GrangerResult$ and $VLGrangerResult$ are true}, then we compare the residual of variable-lag regression $r_{DTW}$ with both$r_Y$ and $r_{YX}$. If $BIC(r_{DTW}) < min(BIC(r_{YX}),BIC(r_{Y}) )$, then we conclude that $X$ causes $Y$ with variable lags, otherwise, $X$ causes $Y$ with a fix lag (line 3 in  Algorithm~\ref{algo:MainFunc}).
\item {\bf If $GrangerResult$ is true but $VLGrangerResult$ is false}, then we conclude that $X$ causes $Y$ with a fix lag (line 4 in Algorithm~\ref{algo:MainFunc}).
\item {\bf If $GrangerResult$ is false but $VLGrangerResult$ is true}, then we conclude that $X$ causes $Y$ with variable lags (line 5 in Algorithm~\ref{algo:MainFunc}).

\item {\bf If both $GrangerResult$ and $VLGrangerResult$ are false}, then we cannot conclude whether $X$ causes $Y$ (line 4 in Algorithm~\ref{algo:MainFunc}).
\squishend

\setlength{\intextsep}{0pt}
\IncMargin{1em}
\begin{algorithm2e}
\caption{Time-lag test function}
\label{algo:MainFunc}
\SetKwInOut{Input}{input}\SetKwInOut{Output}{output}
\Input{ $X,Y$, $\sigma$, $\alpha$, $\delta_{max}$ }
\Output{$XGrangerCausesY$}
\begin{small}
\SetAlgoLined
\nl ($GrangerResult$,$r_{Y},r_{YX}$)=VLGrangerFunc($X,Y$, $\sigma$, $\alpha$, $\delta_{max}$, $FixLag = true$)\;
\nl ($VLGrangerResult$,$r_{Y},r_{DTW}$)= VLGrangerFunc($X,Y$, $\sigma$, $\alpha$, $\delta_{max}$, $FixLag = false$)\;
\nl \uIf{$GrangerResult=true$ }
{
                 \uIf{$VLGrangerResult=true$}
                {
                
               \nl   \uIf{$BIC(r_{DTW}) \ll min(BIC(r_{YX}),BIC(r_{Y}) )$}
                {
                $XGrangerCausesY$ = TRUE-VARIABLE\; 
                }
                 \Else{
                $XGrangerCausesY$ = TRUE-FIXED\; 
                }              
                
                }
                \Else{
                \nl  $XGrangerCausesY$ = TRUE-FIXED\; 
                }
} 
 \Else{
  				\uIf{$VLGrangerResult=true$}
                {
                 \nl  $XGrangerCausesY$ = TRUE-VARIABLE\; 
                }
                 \Else{
                \nl  $XGrangerCausesY$ = NONE\; 
                }
}
\nl return $XGrangerCausesY$\;
\end{small}
\end{algorithm2e}\DecMargin{1em}

Note that we assume the maximum lag value $\delta_{\max}$ is given as input, as it is for all definitions of Granger causality. For practical purposes, a value of a large fraction ({\em e.g., } half) of the length of the time series can be used. However, there is, of course, a computational trade-off between the magnitude of $\delta_{\max}$ and the time it takes to compute Granger causality by almost all methods.

In the next section, we describe the details of VL-Granger function that we use in Algorithm~\ref{algo:MainFunc}: line 1-2.

\subsection{VL-Granger causality operationalization}.

As shown in the previous section, we can operationalize VL-Granger causality by checking for variable-lag emulation. 
Given time series $X,Y$, a similarity threshold $\sigma$, a significance level $\alpha$, the maximum possible lag $\delta_{max}$, and whether we want to check for variable or fixed lag $FixLag$, Algorithm~\ref{algo:VLGrangerCalFunc} reports whether $X$ causes $Y$ by setting $GrangerResult$ to true or false, and by reporting on two residuals  $r_{Y}$ and $r_{YX}$.

First, we compute the residual $r_Y$ of regressing of $Y$ on $Y$'s information in the past (line 1). Then, we regress $Y(t)$ on $Y$ and $X$ past information to compute the residual $r_{YX}$ (line 2). If $BIC(r_{YX}) \ll BIC(r_{Y})$, then $X$ Granger-causes $Y$ and we set $GrangerResult=true$ (line 3). 
If $FixLag$ is true, then we report the result of typical Granger causality. Otherwise, we consider VL-Granger causality (lines 3-7) by computing the emulation relation between $X^{DTW}$ and $Y$ where $X^{DTW}$ is a version of $X$ that is reconstructed through DTW and is most similar to $Y$, captured by $DTWreconsFunc(X,Y)$ which we explain in Section~\ref{sec:DTWRecont}. 

Afterwards, we do the regression of $Y$ on $X^{DTW}$'s past information to compute  residual $r_{DTW}$ (line 4). Finally, we check whether $BIC(r_{DTW}) \ll BIC(r_{Y})$ and the similar value $simValue\geq \sigma$ (line 8-11). If so, $X$ VL-Granger-cause $Y$.  In the next section, we describe the details of how to construct $X^{DTW}$ and how to estimate the emulation similarity value $simValue$.

\setlength{\intextsep}{0pt}
\IncMargin{1em}
\begin{algorithm2e}
\caption{VLGrangerFunc}
\label{algo:VLGrangerCalFunc}
\SetKwInOut{Input}{input}\SetKwInOut{Output}{output}
\Input{$X,Y$, $\delta_{max}$, $\alpha$, $\sigma$, $FixLag$ }
\Output{ $GrangerResult$,$r_{Y},r_{YX}$}
\begin{small}
\SetAlgoLined
 \nl Regress $Y(t)$ on $Y(t-\delta_{max}),\dots, Y(t-1)$, then compute the residual $r_{Y}(t)$\;

  \uIf{$FixLag$ is true}{
   \nl Regress $Y(t)$ on $Y(t-\delta_{max}),\dots, Y(t-1)$ and $X(t-\delta_{max}),\dots, X(t-1)$, then compute the residual  $r_{YX}(t)$\;
  }
 \Else{
 \nl $X^{DTW}$,$simValue$ = DTWReconstructionFunction($X,Y$) \;
 \nl  Regress $Y(t)$ on $Y(t-\delta_{max}),\dots, Y(t-1)$ and $X^{DTW}(t-\delta_{max}),\dots, X^{DTW}(t-1)$, then compute the residual  $r_{DTW}$\;
 \nl $r_{YX}=r_{DTW}$\;
\nl \uIf{$simValue <\sigma$} 
{
 $GrangerResult=false$\;
\nl return $GrangerResult$,$r_{Y},r_{YX} $\;
}

 } 
 
 \nl \uIf{$BIC(r_{YX}) \ll BIC(r_{Y})$  } 
{
\nl $GrangerResult=true$
}
\nl \Else{
\nl $GrangerResult=false$ \;}
\nl return $GrangerResult$,$r_{Y},r_{YX} $\;
\end{small}
\end{algorithm2e}\DecMargin{1em}

\subsection{Dynamic Time Warping for inferring VL-Granger causality.}
\label{sec:DTWRecont}


In this work, we propose to use Dynamic Time Warping (DTW)~\cite{sakoe1978dynamic}, which is a standard distance measure between two time series.  DTW calculates the distance between two time series by aligning sufficiently similar patterns between two time series, while allowing for local stretching (see Figure~\ref{fig:LinearVSArbLag}). Thus, it is particularly well suited for calculating the variable lag alignment. 

Given time series $X$ and $Y$, Algorithm~\ref{algo:DTWReFunc} reports reconstructed time series $X^{DTW}$ based on $X$ that is most similar to $Y$, as well as the emulation similarity $simValue$ between the two series. 
First, we use $DTW(X,Y)$ to find the optimal alignment sequence $\hat{P}=(\Delta_1,\dots,\Delta_t,\dots)$ between $X$ and $Y$, as defined in Definition~\ref{def:AlignSeq}. Efficient algorithms for computing $DTW(X,Y)$ exist and they can incorporate various kernels between points~\cite{Mueen:2016:EOP:2939672.2945383,sakoe1978dynamic}. Then, we use $\hat{P}$ to construct $X^{DTW}$ where $X^{DTW}(t)=X(t-\Delta_t)$. Afterwards, we use $X^{DTW}$ to predict $Y$ instead of using only $X$ information in the past in order to infer a VL-Granger causal relation in Definition~\ref{def:ArbCausalR}. The benefit of using DTW is that it can match time points of $Y$ and $X$ with non-fixed lags (see Figure~\ref{fig:LinearVSArbLag}). Let $\hat{P}=(\Delta_1,\dots,\Delta_t,\dots)$ be the DTW optimal warping path of $X,Y$ such that for any $\Delta_t \in \hat{P}$, $Y(t)$ is most similar to $X(t-\Delta_t)$. 

In addition to finding $X^{DTW}$, $DTWReconstructionFunction$ estimates the emulation similarity $simValue$ between $X,Y$ in line 3. For that, we adopt the measure from~\cite{FLICAtkdd} below: 


\begin{equation}
	\mathrm{s}(\hat{P})=\frac{\sum_{\Delta_t \in \hat{P}}\mathrm{sign}(\Delta_t)}{|\hat{P}|},
	\label{eq:traCorr}
\end{equation}
where $0<\mathrm{s}(\hat{P})\leq 1$ if $X\preceq Y$,  $-1\leq \mathrm{s}(\hat{P})<0$ if $Y\preceq X$, otherwise zero. Since the $\mathrm{sign}(\Delta_t)$ represents whether $Y$ is similar to $X$ in the past ($\mathrm{sign}(\Delta_t) >0$) or  $X$ is similar to $Y$ in the past ($\mathrm{sign}(\Delta_t) <0$), by comparing the sign of $\mathrm{sign}(\Delta_t)$, we can infer whether $Y$ emulates $X$. The function $\mathrm{s}(\hat{P})$ computes the average sign of   $\mathrm{sign}(\Delta_t)$ for the entire time series. 
If $\mathrm{s}(\hat{P})$ is positive, then, on average, the number of times that $Y$ is similar to $X$ in the past is greater than the number of times that $X$ is similar to some values of $Y$ in the past.  Hence, $\mathrm{s}(\hat{P})$ can be used as a proxy to determine whether $Y$ emulates $X$ or vice versa. 

\setlength{\intextsep}{0pt}
\IncMargin{1em}
\begin{algorithm2e}
\caption{ DTWReconstructionFunction}
\label{algo:DTWReFunc}
\SetKwInOut{Input}{input}\SetKwInOut{Output}{output}
\Input{ $X,Y$}
\Output{ $X^{DTW}$, $simValue$}
\begin{small}
\SetAlgoLined
\nl $\hat{P}=(\Delta_1,\dots,\Delta_t,\dots)$ = DTW( $X,Y$) \;
\nl For all $t$, set  $X^{DTW}(t)=X(t-\Delta_t)$ \;
\nl        $simValue =\mathrm{s}(\hat{P})$ \;
Return $X^{DTW}$, $simValue$\;
\end{small}
\end{algorithm2e}\DecMargin{1em}

\section{Experiments}
We measured our framework performance on the task of inferring causal relations using both simulated and real-world datasets. The notations and symbols we use in this section are in Table~\ref{tb:symbloTable}.

\subsection{Experimental setup}
\label{sec:ExpSet}

\begin{table*}[]
\caption{Notations and symbols }
\label{tb:symbloTable}
\begin{small}
\begin{tabular}{|l|p{5.5in}|}
\hline
{\bf Term and notation} & {\bf Description}                                                                                                                                \\ \hline\hline
$T$               & Length of time series.                                                                                                                      \\ \hline
$\sigma$          & Similarity threshold in Definition~\ref{def:follr}.                                                                               \\ \hline
$\delta_{max}$    & Parameter of the maximum length of time delay                                                                                               \\ \hline
$A \indep B$      & $A$ is independent of $B$ according to the independence test in Section~\ref{sec:ExpSet}.                                       \\ \hline
BIC               & Bayesian Information Criterion, which is used as a proxy to compare the residuals of regressions of two time series.                            \\ \hline
$A \prec B$       & $B$ emulates $A$ w.r.t. the threshold $\sigma$.                                                                                             \\ \hline
$\mathcal{N}$       & Normal distribution.                                                                                             \\ \hline
ARMA or A.       & Auto-Regressive Moving Average model.                                                                                             \\ \hline
VG (All)          & Variable-lag Granger causality with F-test:\newline Given time series $X,Y$, $X$ causes $Y$ if 1) $X$ Variable-lag-Granger causes $Y$ using F-test, 2) $X \prec Y$, and 3) $X \indep Y$.            \\ \hline
VG (No F-test)    & Variable-lag Granger causality, no F-test:\newline Given time series $X,Y$, $X$ causes $Y$ if 1) $X$ Variable-lag-Granger causes $Y$ using BIC comparison, 2) $X \prec Y$, and 3) $X \indep Y$. \\ \hline
G                 & Granger causality ~\cite{atukeren2010relationship}                                                                                                                          \\ \hline
CG                & Copula-Granger method ~\cite{liu2012sparse}                                                                                                                    \\ \hline
SIC               &  Spectral Independence Criterion method ~\cite{shajarisales2015telling}                                                                                                                                           \\ \hline
\end{tabular}
\end{small}
\end{table*}

We tested the performance of our method on synthetic datasets, where we explicitly embedded a variable-lag causal relation, as well as on biological datasets in the context of the application of identifying initiators of collective behavior. 

We compared our method, VL-Granger causality (VG) with several existing methods: Granger causality with F-test (G)~\cite{atukeren2010relationship}, Copula-Granger method (CG)~\cite{liu2012sparse}, and Spectral Independence Criterion method (SIC)~\cite{shajarisales2015telling}. We have two different versions of VG, \emph{VG All} and \emph{VG No F-test}. \emph{VG All} consider a relation to be  a causal relation only if it has a strong predictability property (verified by F-test to check $BIC(r_{YX}) \ll BIC(r_{Y})$), strong emulation relation ($simValue \geq \sigma$), and a strong dependency between the effect and the past information of the cause (verified by the independence test that will be discussed in Section~\ref{sec:indTest}). \emph{VG No F-test} is the same as \emph{VG All} except that a causal relation can have a weak predictability property  (verified by the condition $BIC(r_{YX}) < BIC(r_{Y})$ without using F-test).

In this paper, we explore the choice of  $\delta_{max}$ in $\{0.1T,0.2T,0.3T,0.4T\}$ for all methods to analyze the sensitivity of each method, where $T$ is the length of time series, and set $\sigma=0.5$ as default unless explicitly stated otherwise. 

\textbf{Independence test}.
\label{sec:indTest}
To infer Granger causality, we estimate the  regression between $Y$ and the past part of   $X$, to check whether $X$ predicts $Y$. However, if $X$ causes $Y$, then, in addition to predicting, $Y$ must have dependency with some pattern of $X$ in the past. By checking for dependency between $Y$ and an aligned version of $X$, we can have higher confidence that $X$ might indeed be the  cause of $Y$.  
 
 We deploy Hilbert-Schmidt Independence Criterion (HSIC) test~\cite{gretton2008kernel}, which outperforms contingency table and correlation-based tests~\cite{gretton2008kernel}, for testing the dependency between $Y$ and the past of $X$. The null hypothesis $H_0$ is that the given variables are independent. We use $\alpha$ as a significance level needed to reject $H_0$. We check the dependency between $Y$ starting at time $t+1$ with all the shifted versions of $X$ and  $X^{DTW}$, starting from time $1$ to $t$. If $Y$ is simply a shifted version of $X$ or a non-fixed-lag distortion of $X$, then $Y$ must have a dependency with some version of $X$ and HSIC must reject $H_0$.  We use HSIC with bootstrap resampling to estimate the test threshold. 

\subsection{Datasets}

\textbf{Synthetic data: pairwise level}.
The main purpose of the synthetic data is to generate settings that explicitly illustrate the difference between the original Granger causality method and the proposed variable-lag approach. We generate pairs of time series for which the fixed-lag Granger causality methods would fail to find a relationship but the variable-lag approach would find the intended relationships. 

We generated a set of synthetic time series of 200 time steps. We generated two sets of pairs of time series $X$ and $Y$.
First, we generated $X$ either by drawing the value of each time step from a normal distribution with zero mean and a constant variance ($X(t)\sim\mathcal{N}$) or by Auto-Regressive Moving Average model (ARMA or A.) with $X(t)=0.9\mu + 0.1X(t-1)$.  

The first set we generated was of explicitly related pairs of time series $X$ and $Y$, where  $Y$ emulates $X$ with some  time lag $\Delta \in [1,20]$ ($X\prec Y$). One way to ensure lag variability is to ``turn off" the emulation for some time. In our data, $Y$ remains constant between 110th and 150th time steps. This makes $Y$ a variable-lag follower of $X$. Figure~\ref{fig:FixVsVarLagExample} shows examples of the generated time series.

The second set of time series pairs $X$ and $Y$ were generated independently and as a result have no causal relation. We used these pairs to ensure that our method does not infer spurious relations. 

We set the significance level for both F-test and independence test at $\alpha=0.01$. We  considered there to be a causal relation only if $simValue\geq \sigma$ for our method. 

\textbf{Synthetic data: group level}. This experiment explores the ability of causal inference methods to retrieve {\em multiple}  causes of a time series $Y_{ij}$, which is generated from multiple time series $X_i,X_j$. Fig.~\ref{fig:SynHSiblingRes} shows the ground truth causal graph we used to generate simulated datasets. The edges represent causal directions from the cause time series (e.g. $X_1$) to the effect time series (e.g. $Y_1$). $Y_{ij}$ represents the time series generated by $agg(\{X'_i,X'_j\})$, where $X_i\prec X'_i$ and $X_j\prec X'_j$ with some fixed lag $\Delta\in[1,20]$. The task is to infer edges of this causal graph from the time series. We generated time series for each generator model 100 times.

\begin{figure}
\centering
\includegraphics[width=.48\columnwidth]{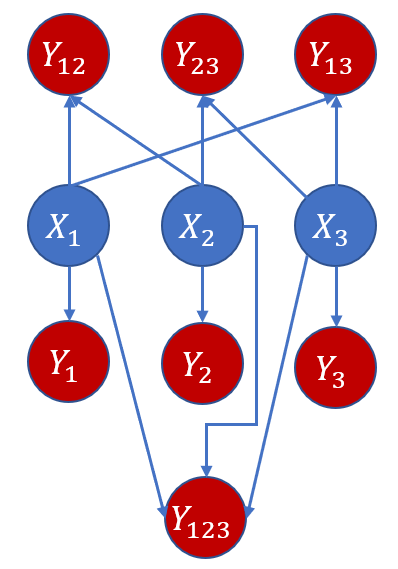}
\caption{ The causal graph where the edges represent causal directions from the cause time series (e.g. $X_1$) to the effect time series (e.g. $Y_1$). $Y_{ij}$ represents a time series generated by $agg(\{X'_i,X'_j\})$, where $X_i\prec X'_i$ with some fixed lag $\Delta$. }
\label{fig:SynHSiblingRes}
\end{figure}

\begin{figure}
\centering
\includegraphics[width=1\columnwidth]{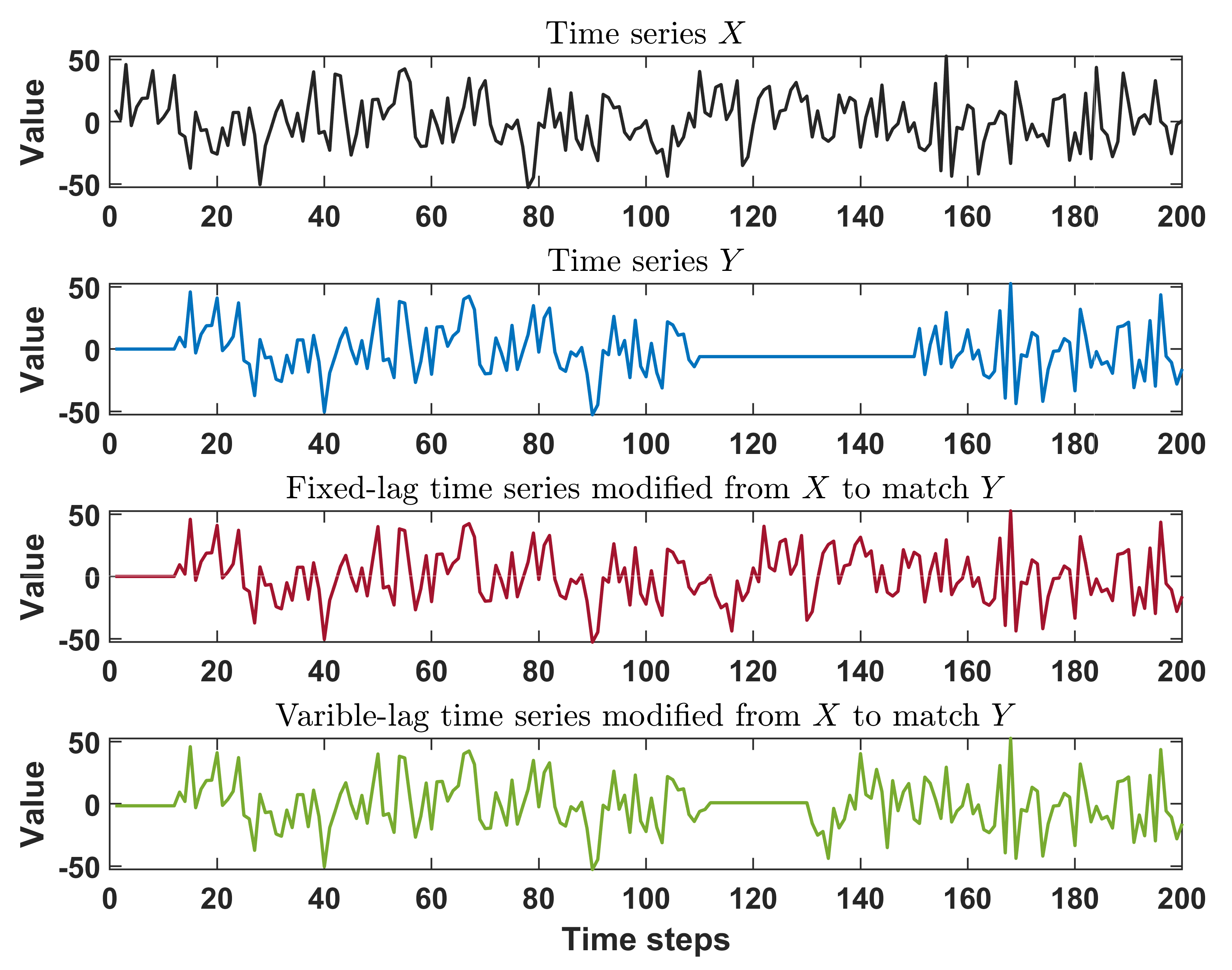}
\caption{ The comparison between the original time series $X$, variable-lag follower $Y$, fixed-lag time series modified from $X$ to match $Y$, and variable-lag time series modified from $X$ to match $Y$. The traditional Granger causality uses only fixed-lag version of $X$ to infer whether $X$ causes $Y$, while our approach uses both versions of $X$ to determine the causality between $X,Y$. Both $X,Y$ are generated from $\mathcal{N}$. $Y$ remains constant from time 110 to 150, which makes it a variable-lag follower of $X$. }
\label{fig:FixVsVarLagExample}
\end{figure}

\textbf{Schools of fish}. We used the dataset of golden shiners (\emph{Notemigonus crysoleucas}) that is publiclly available. The dataset has been collected for the study of information propagation over the visual fields of fish \cite{strandburg2013visual}. There are $24$ coordination events in this dataset. Each coordination event consists of two-dimensional time series of fish movement that are recorded by video. The time series of fish movement are around 600 time steps.  The number of fish in each dataset is around 70 individuals, of which 10 individuals are ``informed" fish who have been trained to go to a feeding site. Trained fish lead the group to feeding sites while the rest of the fish just follow the group.  We represent the dataset as a pair of aggregated time series: $X$ being the aggregated time series of the directions of trained fish and $Y$ being the aggregated time series of the directions of untrained fish. The task is to infer whether $X$ (trained fish) is  a cause of $Y$ (the rest of the group).

\textbf{Troop of baboons}. We used another publicly available dataset of animal behavior, the movement of a troop of olive baboons (\emph{Papio anubis}). The dataset consists of GPS tracking information from 26 members of a troop, recorded at 1 Hz from 6 AM to 6 PM between August 01, 2012 and August 10, 2012. 
The troop lives in the wild at the Mpala Research Centre, Kenya \cite{crofoot2015data,strandburg2015shared}. For the analysis, we selected the 16 members of the troop that have GPS information  available for 10 consecutive days, with no missing data. We selected  a set of trajectories of lat-long coordinates from a highly coordinated event that has the length of 600 time steps (seconds) for each baboon.  This known coordination event is on August 02, 2012 in the morning, with the baboon ID3 initiating the movement, followed by the rest of the troop~\cite{FLICAtkdd}. Again, the goal is to infer ID3 (time series $X$) as the cause of the movement of the rest of the group (aggregate time series $Y$).
In this experiment  we set $\delta_{max} = 120$ seconds, informed by the biological meaning and expertise. 

\subsection{Time complexity and running time}

\begin{table}[]
\caption{Running time of our approach with varying time series length $T$ and maximum time delay $\delta_{max}$.}
\label{tab:runningTime}
\begin{center}
\begin{tabular}{c|c|c|}
\cline{2-3}
                                            & $T=500$ & $T=2000$ \\ \hline
\multicolumn{1}{|c|}{$\delta_{max}=0.05T$} & 22 sec  & 254 sec  \\ \hline
\multicolumn{1}{|c|}{$\delta_{max}=0.10T$} & 59 sec  & 505 sec  \\ \hline
\multicolumn{1}{|c|}{$\delta_{max}=0.2T$}  & 93 sec  & 1680 sec \\ \hline
\end{tabular}
\end{center}
\end{table}

The main cost of computation in our approach is DTW. We used the ``Windowing technique'' for the search area of warping ~\cite{keogh2001derivative}. The main parameter for windowing technique is the maximum time delay $\delta_{max}$. Hence, the time complexity of our approach is $\mathcal{O}(T\delta_{max})$.  Table~\ref{tab:runningTime} shows the running time of our approach on time series with the varying length ($T\in \{500,2000\}$) and maximum time delay ($\delta_{max} \in \{0.05T,0.1T,0,2T\}$).

\section{Results}
We report the results of our proposed approach and the standard Granger causality method on both synthetic and biological data. We also explore how the performance of the methods depends on the basic parameter, $\delta_{\max}$.

\subsection{Synthetic data: pairwise level}

\begin{table}[]
\caption{Average accuracy of inferring causal direction from various cases. Each column represents a method. Each row represents a model.  ``$\mathcal{N}$:$X$" means $X$ was generated from a normal distribution and ``A.:$X$" means $X$ was generated from ARMA model. $X\prec Y$ means $X$ causes $Y$ by an emulation relation and $X\nprec Y$ means no causal relation.  We varied $\delta_{max}$ from $10\%$ to $40\%$ of time series length $T$ and reported the average. }
\label{tb:syndataRes}
\begin{small}
\begin{tabular}{c|c|c|c|c|c|}
\cline{2-6}
\multicolumn{1}{l|}{}                                                                  & \cellcolor[HTML]{C0C0C0}VG  & \cellcolor[HTML]{C0C0C0}VG        & \cellcolor[HTML]{C0C0C0}G & \cellcolor[HTML]{C0C0C0}CG & \cellcolor[HTML]{C0C0C0}SIC \\
\multicolumn{1}{l|}{}                                                                  & \cellcolor[HTML]{C0C0C0}All & \cellcolor[HTML]{C0C0C0}No F-test & \cellcolor[HTML]{C0C0C0}  & \cellcolor[HTML]{C0C0C0}   & \cellcolor[HTML]{C0C0C0}    \\ \hline
\multicolumn{1}{|c|}{$\mathcal{N}$:$X\prec Y$}                                   & 1.00                        & 1.00                              & 1.00                      & 0.79                       & 0.70                        \\ \hline
\rowcolor[HTML]{EFEFEF} 
\multicolumn{1}{|c|}{\cellcolor[HTML]{EFEFEF}$\mathcal{N}$:$X\nprec Y$} & 1.00                        & 1.00                              & 0.92                      & 0.90                       & 0.51                        \\ \hline
\multicolumn{1}{|c|}{A.:$X\prec Y$}                                              & 0.97                        & 0.97                              & 1.00                      & 0.82                       & 0.65                        \\ \hline
\rowcolor[HTML]{EFEFEF} 
\multicolumn{1}{|c|}{\cellcolor[HTML]{EFEFEF}A.:$X\nprec Y$}            & 0.95                        & 0.79                              & 0.78                      & 0.85                       & 0.49                        \\ \hline
\multicolumn{1}{|c|}{$\mathcal{N}$:$X$,  A.:$Y$}                                       & 0.99                        & 0.99                              & 0.95                      & 0.91                       & 0.68                        \\ \hline
\rowcolor[HTML]{EFEFEF} 
\multicolumn{1}{|c|}{\cellcolor[HTML]{EFEFEF}Group $\mathcal{N}$}                      & 1.00                        & 1.00                              & 1.00                      & 0.12                       & 0.53                        \\ \hline
\multicolumn{1}{|c|}{Group A.}                                                         & 1.00                        & 1.00                              & 1.00                      & 0.30                       & 0.50                        \\ \hline
\end{tabular}
\end{small}
\end{table}


Table~\ref{tb:syndataRes} (1st-5th rows) shows the results of the accuracy of inferring causal relations and directions. For each row, we repeated the experiment 100 times  on different simulated datasets and computed the accuracy and reported the mean. The result shows that our methods, VG (all) and (No F-test),  performed better than the rest of other methods. Moreover, we also investigate the sensitivity of varying the value of the $\delta_{max}$ parameter for all methods. We  aggregated the accuracy of inferring causal direction from various cases that have the same $\delta_{max}$ value and report the result. The result in Fig.~\ref{fig:SynAvgAccDeltaMax} shows that our approaches: VG (ALL) and VG (No F-test), can maintain the high accuracy throughout the range of the  values of $\delta_{max}$.  

\begin{figure}
\centering
\includegraphics[width=0.8\columnwidth]{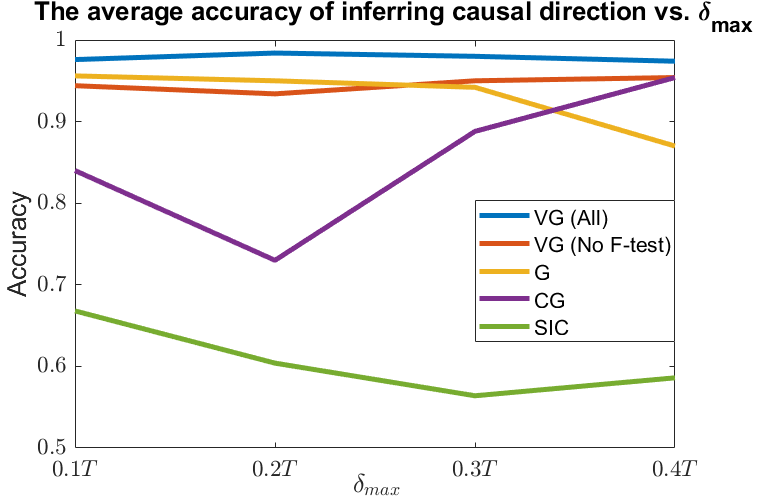}
\caption{ Average accuracy of inferring causal direction as a function of $\delta_{max}$. $x$-axis represents the value of $\delta_{max}$ as a fraction of the time series length $T$ and $y$-axis is the average accuracy. }
\label{fig:SynAvgAccDeltaMax}
\end{figure}

\subsection{Synthetic data: group level}

 Table~\ref{tb:GrHSiblingSynRes} shows the result of causal graph inference. The VG (All),  VG (No F-test), and G (1st-3rd rows) performed the best overall with the highest F1 score in the normal generator model ($\mathcal{N}$).  However, VG (All) and VG (No F-test) have higher F1 score than G in ARMA generator model. This result reflects the fact that our approaches can handle complicated time series in the Granger causal inference task better than the rest of other methods. In addition, we aggregated $X=agg(\{X_1,X_2,X_3\})$ and the rest of time series $Y=agg(\{Y_1,Y_2,\dots,Y_{123}\})$, then we measured the ability of methods to infer that $X$ is  a cause of $Y$. The results, which are in the ``Group $\mathcal{N}$" and ``Group A." rows in Table~\ref{tb:syndataRes}, show that  VG (All),  VG (No F-test), and G performed well in this task, while CG and SIC failed to infer causal relations. 

\begin{table}[t]
\caption{The results of the precision, recall, and F1-score values of edges inference of causal graph in Fig.~\ref{fig:SynHSiblingRes}.  Each row is a method and each column is a result from different generator model of initiators ($X_1,X_2,X_3$ in the causal graph).  We varied $\delta_{max}$ from $10\%$ to $40\%$ of time series length $T$ and reported the average.}
\label{tb:GrHSiblingSynRes}
\begin{small}
\begin{tabular}{c|c|c|c|c|c|c|}
\cline{2-7}
\multicolumn{1}{l|}{}                                  & \multicolumn{3}{c|}{\cellcolor[HTML]{C0C0C0} $\mathcal{N}$} & \multicolumn{3}{c|}{\cellcolor[HTML]{C0C0C0}ARMA} \\ \cline{2-7} 
\multicolumn{1}{l|}{}                                  & Prec.           & Rec.            & F1              & Prec.           & Rec.           & F1             \\ \hline
\rowcolor[HTML]{EFEFEF} 
\multicolumn{1}{|c|}{\cellcolor[HTML]{EFEFEF}VG (All)} & 0.98 &	0.90 &	0.93 &	0.66 &	0.95	& 0.77
           \\ \hline
\multicolumn{1}{|c|}{VG (No F-test)}                   &0.96	&0.90	&0.93	&0.57	&0.95	&0.71           \\ \hline
\rowcolor[HTML]{EFEFEF} 
\multicolumn{1}{|c|}{\cellcolor[HTML]{EFEFEF}G}  &0.91	&0.97	&0.93	&0.38	&0.71	&0.49

  \\ \hline
\multicolumn{1}{|c|}{CG}                               & 0.55	&0.68	&0.59	&0.42	&0.64	&0.47
           \\ \hline
\rowcolor[HTML]{EFEFEF} 
\multicolumn{1}{|c|}{\cellcolor[HTML]{EFEFEF}SIC}      & 0.15	&0.55	&0.23	&0.14	&0.53	&0.23
           \\ \hline
\end{tabular}
\end{small}
\end{table}

\subsection{Animal data: schools of fish}



\begin{table}[t]
\caption{The result of inferring causal direction that $X$ (trained fish) causes $Y$ (untrained fish) behavior, from 24 traces of coordinated movements. Each element represents a number of times (max 24) that each method inferred the correct causal direction (higher is better).}
\label{tb:fish24tr}
\begin{small}
\begin{tabular}{|c|c|c|c|c|c|}
\hline
\rowcolor[HTML]{C0C0C0} 
\multicolumn{1}{|l|}{\cellcolor[HTML]{C0C0C0}}             & VG  & VG        & G & CG & SIC \\
\rowcolor[HTML]{C0C0C0} 
\multicolumn{1}{|l|}{\cellcolor[HTML]{C0C0C0}$\delta_{max}$} & All & No F-test &   &    &     \\ \hline
$0.1T$                                                     & 14  & 19        & 2 & 17 & 19  \\ \hline
\rowcolor[HTML]{EFEFEF} 
$0.2T$                                                     & 18  & 19        & 2 & 11 & 21  \\ \hline
$0.3T$                                                     & 16  & 18        & 5 & 11 & 20  \\ \hline
\rowcolor[HTML]{EFEFEF} 
$0.4T$                                                     & 18  & 18        & 5 & 11 & 19  \\ \hline
\end{tabular}
\end{small}
\end{table}


Table~\ref{tb:fish24tr} shows the result of inferring causal direction that $X$ (trained fish) cause $Y$ (untrained fish) group behavior, as observed from 24 traces of coordinated movements. Granger (G) performed poorly while VG, CG, and SIC performed well on this task. This result reflects the fact that the causal relation between $X$ and $Y$ have highly variable lags of influence between causes and effects. Granger causality typically dismisses the existence of this type of causation and, therefore, misses the behavior. 

\subsection{Animal data: a troop of baboons}

\begin{figure}
\centering
\includegraphics[width=1\columnwidth]{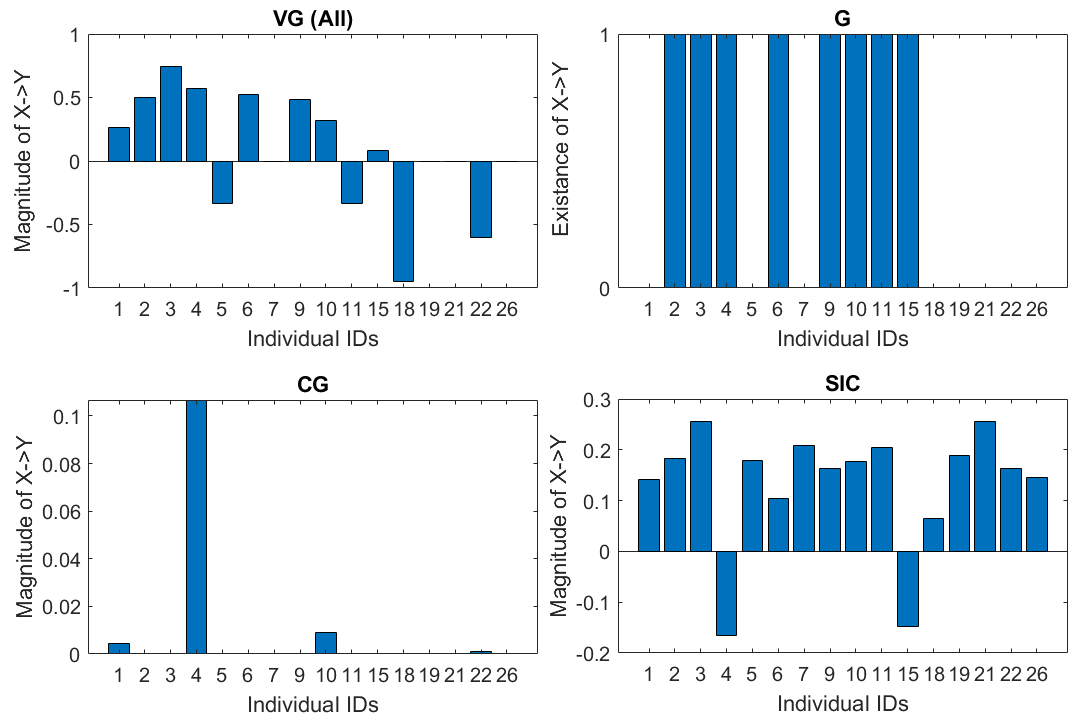}
\caption{ Initiator inference from the time series of coordinated movement in baboons with $\delta_{max}=120$. $ID3$ is the only true initiator in the dataset. Each plot represents the inferred level of each individual (x-axis) being the cause of the group movement (y-axis: higher value is better). As shown, VG (all) is able to uncover the true initiator as the cause of the group's movement, unlike G, CG and SIC (not the $y$-axis scale). 
}
\label{fig:BaboonAug2Res}
\end{figure}


Recall, that in this dataset ID3 is the true initiator of the group's behavior during the observed interval of coordinated behavior.
Fig.~\ref{fig:BaboonAug2Res} shows the result of initiator inference in this coordinated interval. Only VG (all) in the top-left  of the Fig.~\ref{fig:BaboonAug2Res} gave the correct initiator (ID3) the highest weight. G did not differentiate among a large group of individuals as the potential cause. CG gave a week evidence of 0.1 to ID4 (wrong) as the initiator. Finally, SIC gave equal weak support of less than 0.3 to both ID3 and ID21 (false positive). This result shows the robustness of our approach over other methods in the real-world noisy datasets. 
According to ~\cite{shajarisales2015telling}, this implies that ID3 has a spectral independence from the collective time series $Y=agg(\mathcal{U}\setminus \{ID3\})$, while $Y$ has spectral dependence with ID3. This result shows that not only can we gain insight into whether $X$ VG-causes $Y$ from our method, but we can use SIC (and perhaps other methods) to confirm the causal properties between $X$ and $Y$.

\section{Limitation of VL-Granger causality and suggestions}
One of the main input of Granger causality is the max lag variable $\delta_{max}$. There is a report in~\cite{bruns2019lag} that Granger causality might face the overfitting issue when $\delta_{max}$ is inappropriately chosen. The VL-Granger causality also has the same issue. The cross-validation technique might be used to combat this issue when there is no ground-truth of $\delta_{max}$ is available. Moreover, since VL-Granger causality is the generalization of traditional Granger method, it can easily face the issue of overfitting - false positive discovery of causality relation. This situation is similar to the case of fitting linear vs. non-linear regression. The more generalized model seems to suffer from the issue of overfitting easier than a simpler model. We suggest that if the variance of VL-Granger is not different from traditional method, the conclusion of traditional Granger approach should be used. 

Regarding the scalability aspect, even though VL-Granger causality is a generalization of the traditional method, it requires more resources in both time and space due to the need of inferring variable lags. The efficient approach of inferring variable lags should be developed in future. 
\section{Conclusions}
In this work, we proposed a method to infer Granger causal relations in time series where the causes influence effects with arbitrary time delays, which can change dynamically. We formalized a new Granger causal relation, proving that it is a true generalization of the traditional Granger causality. We demonstrated on both carefully designed synthetic datasets and noisy biological data that the new causal relation can address the arbitrary-time-lag influence between cause and effect, while the traditional Granger causality cannot. Moreover, in addition to  improving and extending Granger causality,  our approach can be applied to infer leader-follower relations, as well as the dependency property between cause and effect. We have shown that, in many situations, the causal relations between time series do not have a lock-step connection of a fixed lag that the traditional Granger causality assumes. Hence, traditional Granger causality misses true existing causal relations in such cases, while our method correctly infers them. Our approach can be applied in any domain of study where the causal relations between time series is of interest. In future work, we plan to explore non-linear relations between time series using non-linear kernels. We also provide the source code at~\cite{ShareSourcecode}.



%
\balance
\bibliographystyle{IEEEtran}


\end{document}